\documentclass[10pt,twocolumn,letterpaper]{article}

\usepackage[pagenumbers]{iccv} 

\usepackage{microtype}
\usepackage{graphicx}
\usepackage{booktabs} 
\usepackage{amsmath}
\usepackage{amssymb}
\usepackage{mathtools}
\usepackage{amsthm}
\usepackage{amsmath}
\usepackage{amssymb}
\usepackage{algorithm}
\usepackage{algorithmic}
\usepackage{multirow}
\usepackage{mathtools}
\usepackage{gensymb} 
\usepackage[accsupp]{axessibility}

\theoremstyle{plain}
\newtheorem{theorem}{Theorem}[section]
\newtheorem{lemma}[theorem]{Lemma}

\newtheorem{assumption}[theorem]{Assumption}
\newtheorem{example}[theorem]{Example}

\newcommand{\methodname}{\textsc{SEAL}}
\newcommand{\attackname}{\textsc{Cat Attack}}

\newcommand\gensimi{\stackrel{\mathclap{\normalfont\mbox{$s_i$}}}{\sim}}

\newcommand\gensimvec{\stackrel{\mathclap{\normalfont\mbox{$s$}}}{\sim}}
\newcommand\sign{\textnormal{sign}}

\definecolor{iccvblue}{rgb}{0.21,0.49,0.74}
\usepackage[pagebackref,breaklinks,colorlinks,allcolors=iccvblue]{hyperref}

\def\paperID{9711} 
\def\confName{ICCV}
\def\confYear{2025}

\title{SEAL: Semantic Aware Image Watermarking}

\author{Kasra Arabi, R. Teal Witter, Chinmay Hegde, Niv Cohen \\
New York University}

\begin{document}
\maketitle
\begin{abstract}

Generative models have rapidly evolved to generate realistic outputs. However, their synthetic outputs increasingly challenge the clear distinction between natural and AI-generated content, necessitating robust watermarking techniques to mark synthetic images.
Watermarks are typically expected to preserve the integrity of the target image, withstand removal attempts, and prevent unauthorized insertion of the watermark pattern onto unrelated images.
To address this need, recent methods embed persistent watermarks into images produced by diffusion models using the initial noise of the diffusion process.
Yet, to do so, they either distort the distribution of generated images or require searching a large dictionary of candidate noise patterns for detection.

In this paper, we propose a novel watermarking method that embeds semantic information about the generated image into the noise pattern, enabling a distortion-free watermark that can be verified without requiring a database of key patterns. Instead, the key pattern can be inferred from the semantic embedding of the image using locality-sensitive hashing.
Furthermore, conditioning the watermark detection on the original image content improves its robustness against forgery attacks. To demonstrate that, we consider two largely overlooked attack strategies: (i) an attacker extracting the initial noise and generating a novel image with the same pattern; (ii) an attacker inserting an unrelated (potentially harmful) object into a watermarked image, while preserving the watermark. We empirically validate our method's increased robustness to these attacks. 
Taken together, our results suggest that content-aware watermarks can mitigate risks arising from image-generative models.
Our code is available at \url{https://github.com/Kasraarabi/SEAL}.
\end{abstract}    
\section{Introduction}
\label{sec:intro}
The growing capabilities of generative models pose risks to society, including misleading public opinion, violating privacy or intellectual property, and fabricating legal evidence \cite{jaidka2025misinformation, solaiman2023evaluating, bird2023typology}. Watermarking methods aim to mitigate such risks by allowing the detection of synthetically generated content. 

Yet, many conventional watermarking techniques lack robustness against adversaries who attempt to remove them using regeneration attacks powered by recent generative models~\cite{fernandez2023stable, zhao2025invisible, an2024waves}. To address this, new watermarking techniques leveraging similar advances in generative models offer an increased robustness against such attacks \cite{arabi2024hidden, wen2023tree, yang2024gaussian}. Namely, these methods embed a watermarking pattern in the initial noise used by a diffusion model. These patterns have been shown to be more robust against existing removal attacks.

\begin{figure}[t]
    \centering
    \includegraphics[width=\columnwidth]{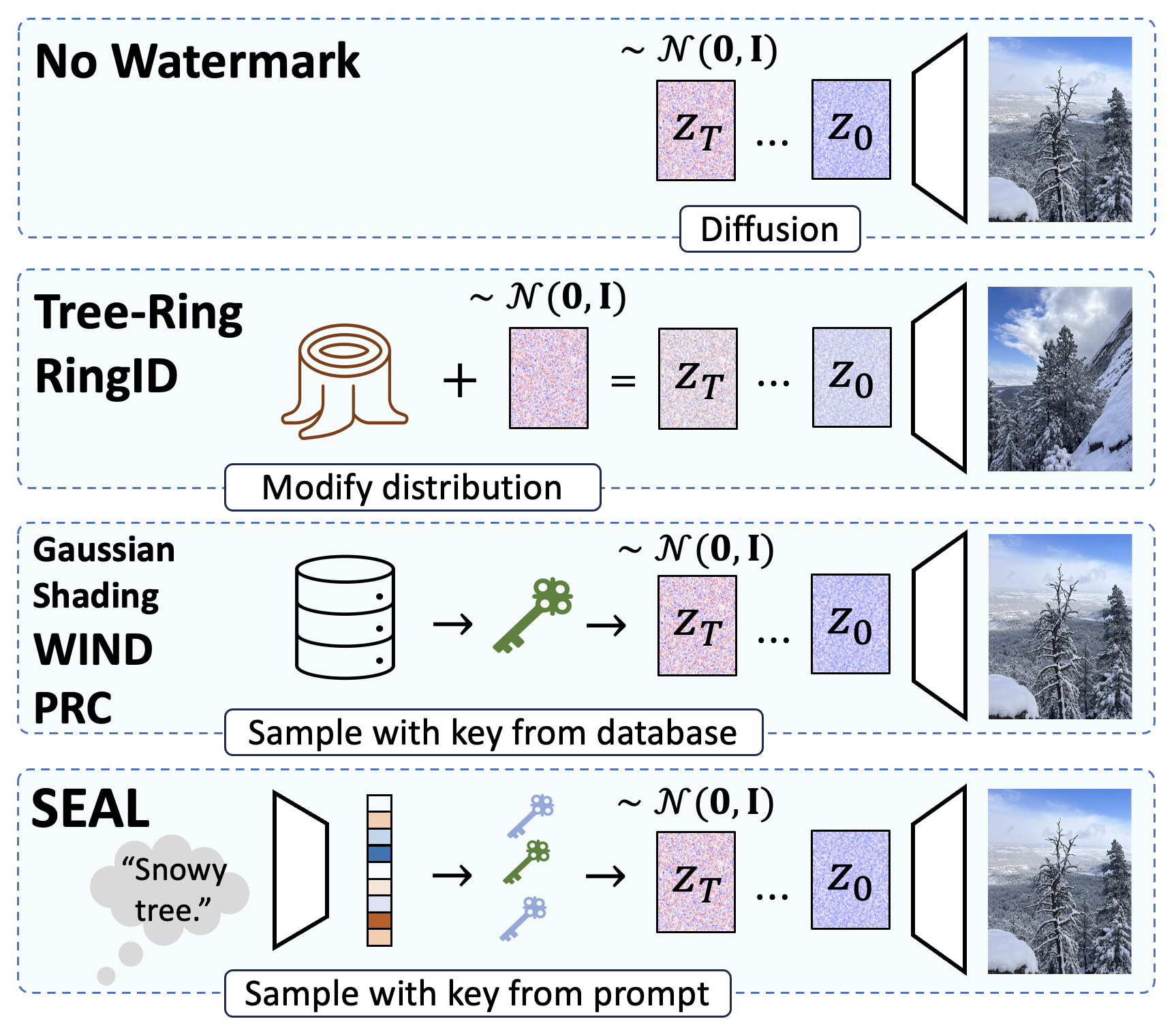}
    \caption{\textit{\textbf{Illustration of different watermarking frameworks using the initial noise of diffusion models.
No Watermark:} A diffusion model maps pure Gaussian noise to an image.  
\textbf{Tree-Ring:} A pattern is added to the initial noise, modifying the distribution of generated images in a detectable way.  
\textbf{Key-Based Watermarking:} A key is sampled to generate distortion-free images linked to the key.  
\textbf{Ours (\methodname):} The initial noise is conditioned on multiple keys derived from the image's semantic embedding, with each key influencing a different patch.}}
    \label{fig:fig1}
\end{figure}

However, existing watermarks that utilize the diffusion model initial noise tend to be vulnerable to other attacks aiming to ``steal" the watermark and apply it to images unrelated to the watermark owners~\cite{yang2024can,muller2025black,jain2025forging}. Some of these \textit{watermark forgery} attacks can be evaded by using a distortion-free watermark - generating watermarked images from a similar distribution to the distribution of non-watermarked images; therefore exposing less information about the watermark identity. Even so, keeping track of a very large number of watermark identities requires maintaining a database of used noises, and might still be forgeable by other attacks~\cite{muller2025black,jain2025forging}. 

To address these challenges, we introduce \methodname ~-~ \textit{Semantic Embedding for AI Lineage}, a method that embeds watermark patterns directly tied to image semantics. Our approach enables direct watermark detection from image samples and offers the following key properties:
(i) \textit{Distortion-free:} As in previous works, we utilize pseudo-random hash functions to generate an initial noise that is similar to the noise used by non-watermarked models, ensuring a similar distribution of generated images.
(ii) \textit{Robust to regeneration attacks:} Similar to prior watermarking methods based on DDIM inversion, our approach demonstrates resilience against regeneration-based removal attempts \cite{zhao2025invisible}.
(iii) \textit{Correlated with image semantics:} The applied watermark encodes semantic information from the image. 
(iv) \textit{Independent of a historical database:} Our approach embeds watermarks without requiring access to a database of used noise patterns.

Our key insight is that we can encode semantic information about the image content in a distortion-free watermark by embedding a semantic encoding of the generation directly into the initial noise. Namely, we may use an encoding of the requested image semantics to seed different pseudo-random patches that compose the initial noise. We ensure the encoded embedding correlates strongly with the resulting image content, and not just with the prompt, which is important since the prompt is not available during detection. At detection time, our approach identifies an image as watermarked only when the watermark pattern is both present and properly correlated with the semantics of the given image. We describe in detail our watermarking technique in \Cref{sec:method}.

Correlating our watermarking algorithm to the image semantics also allows us to resist forgery attempts that are challenging for many existing approaches. An attacker attempting to forge our watermark onto unauthorized content would alter the image's semantic embedding, breaking its correlation with the embedded pattern and rendering the watermark invalid.

One mostly overlooked attack involves an adversary altering only small portions of a watermarked image while preserving the rest of its content. In such cases, the attacker can manipulate the image to be offensive, illegal, or damaging to the watermark owner's reputation, all while the original watermark remains detectable. We term this attack the \textit{\attackname}, as the attacker may add an object to the image (e.g., a cat) and expect the watermark to persist. We evaluate the potential damage of such tamperings and demonstrate that our method provides robustness against both the \attackname ~and other forgery attempts, even by adversaries who obtain accurate copies of our initial noise. Our experiments confirm our method's effectiveness against these novel threats as well as previously studied attack vectors.

\begin{figure*}[t]
    \centering
    \includegraphics[width=0.9\textwidth]{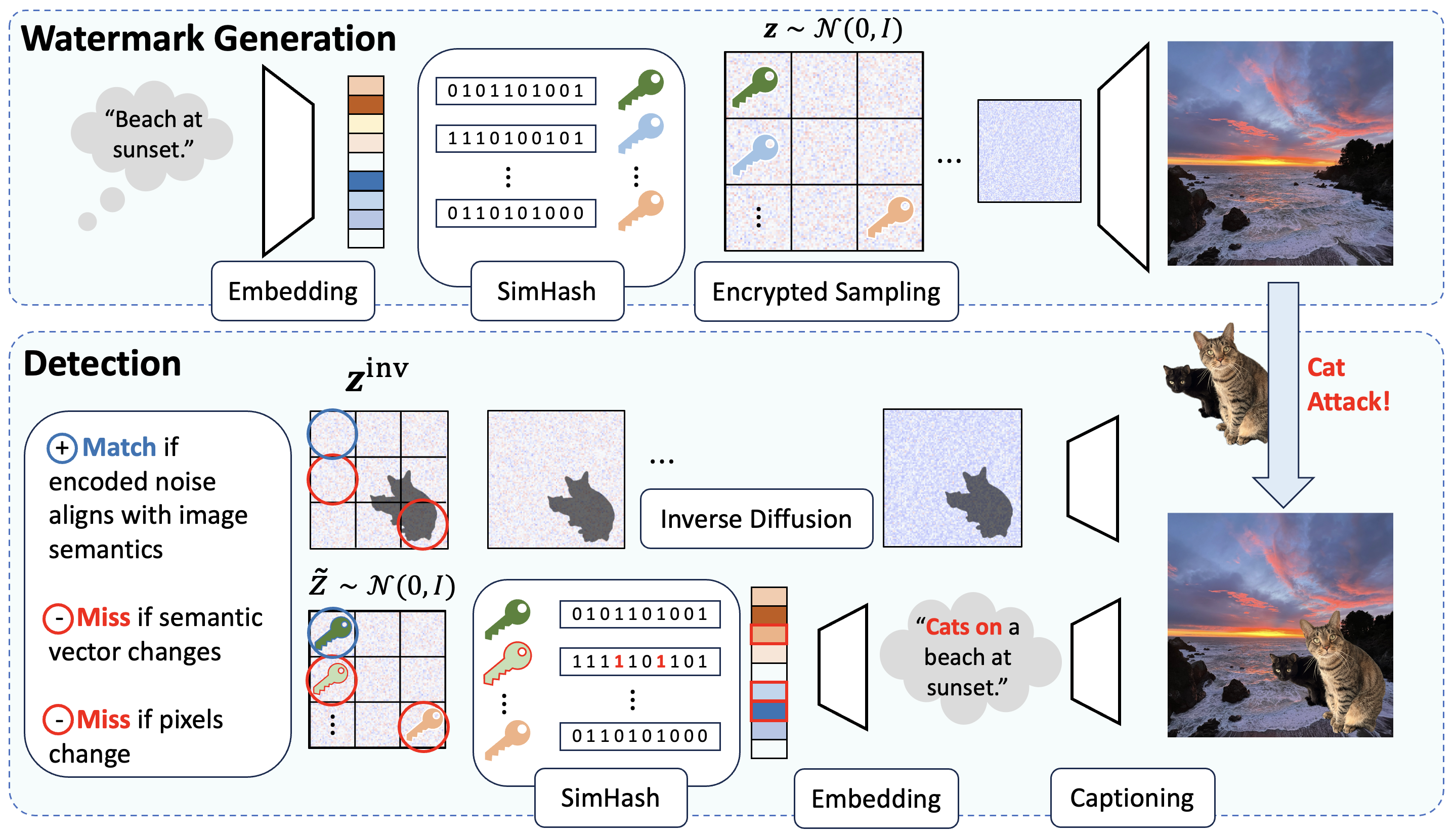}
    \caption{\textit{\textbf{Illustration of the \methodname~watermarking framework for diffusion models using semantic-aware noise patterns.} 
    \textbf{Watermark Generation:} A textual prompt (e.g., ``Beach at sunset.") is first embedded into a semantic space. The embedding is then processed using \textbf{SimHash} to generate discrete keys, which are used in \textbf{Encrypted Sampling} to choose the initial noise $\mathbf{z} \sim \mathcal{N}(0, I)$. The watermarked noise then undergoes standard diffusion to generate the final image.  \textbf{Watermark Detection:} The image is captioned to obtain an embedding, which is then processed by \textbf{SimHash} to generate a reference noise, similarly to watermark generation. This noise remains correlated with the initial noise used during generation as long as the image semantics remain unchanged. The initial noise is also estimated directly through \textbf{Inverse Diffusion} to approximate the actual initial noise used during its generation. If there are insufficient matches between the reference noise and the noise obtained from inversion, the watermarking framework flags the image as non-watermarked. If a key match is found but the image is still deemed suspicious, a detailed inspection of the patches can be performed to identify local edits.
}}
    \label{fig:method_fig2}
\end{figure*}

\noindent\textbf{Our contributions are as follows:}  

\begin{itemize}  
    \item We propose \methodname, a semantic-aware database-free watermarking method that integrates image semantics into the watermark, ensuring it becomes invalid under severe semantic changes.  
    
    \item We investigate the \attackname, highlighting the risks of local edits to watermark owners and assessing their potential impact.

    \item We empirically demonstrate the effectiveness of our watermark against various attacks, including its resistance to adversarial edits.  
\end{itemize}

\section{Related Works}
\label{sec:related_works}
Recent research on image watermarking can be broadly categorized into post-processing and in-processing approaches, each offering distinct trade-offs between quality, robustness, and deployment practicality \citep{an2024waves}. We cover here \textit{In-Processing Methods}, and for \textit{Post-Processing Methods} refer to~\Cref{app:related_works}.

\paragraph{In-Processing Watermarking Methods.}

In-processing approaches integrate the watermark directly within the image generation process. 
Some methods modify the generative model by fine-tuning specific components, as demonstrated in Stable Signature \cite{fernandez2023stable,zhang2024editguard,sander2024watermark}. An alternative class of techniques manipulates the initial noise of the generation process, thereby embedding the watermark without extensive model retraining. For example, Tree-Ring \cite{wen2023tree} embeds a Fourier-domain pattern into the initial noise, which can be detected through DDIM inversion \cite{song2020denoising}, while RingID \cite{ci2024ringid} extends this idea to support multiple keys. Other notable methods include Gaussian Shading, which produces a unique key for each watermark owner \cite{yang2024gaussian}, PRC that leverages pseudo-random error-correcting codes for computational undetectability \cite{gunn2024undetectable}, and WIND, which employs a two-stage detection process to enables a very large number of keys \cite{arabi2024hidden}.

\paragraph{Locally Sensitive Hashing in High-Dimensional Spaces.}  
Recent advances in approximate nearest neighbor (ANN) search have increasingly relied on the power of Locally Sensitive Hashing (LSH) to address the challenges of dealing with high-dimensional data. Originally introduced by Indyk and Motwani \cite{indyk1998approximate} and further refined by Gionis et al. \cite{gionis1999similarity}, LSH employs randomized hash functions that ensure similar data points are mapped to the same bucket with high probability. For a hash family \( \mathcal{H} \), the collision probability is given by  
\[
P(h(x) = h(y)) \approx \text{similarity}(x, y), \quad h \in \mathcal{H}.
\]
Subsequent improvements by Datar et al. \cite{datar2004locality} and Andoni and Indyk \cite{andoni2008near} have enhanced both the efficiency and robustness of LSH methods, making them key for large-scale, high-dimensional search tasks.

\section{\methodname: Semantic Aware Watermarking}
\label{sec:method}

\subsection{Motivation}


Watermarking methods suffer from an inherent trade-off: a watermark that is harder to remove is also easier to attach to unrelated generations, compromising the reputation of the watermark owner \cite{bird2023typology}. One suggested solution to overcome this trade-off, might be maintaining a database of past generations, such that the owner could compare a seemingly watermarked image to the actual past generations. Yet, this solution is not without its problems. First, maintaining and searching a rapidly growing database, which expands with each new generation, can be challenging. Second, safeguarding the database itself may pose security risks. Finally, in various use cases, the watermark owner may not only wish to detect if an image is watermarked but also provide to a third party evidence that it was. We therefore turn to suggest a watermarking scheme that is hard to remove, hard to forge, and does not rely on maintaining a database of past generations.

Our core idea is to use a distortion-free initial noise pattern not only to indicate the origin of the image but also to encode which semantic information the image may contain.
We do so in three stages (see also \Cref{fig:method_fig2}): (i) \textit{Semantic Embedding} – we obtain a vector representing the expected semantic content in each generated image (ii) \textit{SimHash Encoding} – we encode the semantic vector using a set of multi-bit hash functions (iii) \textit{Encrypted Sampling} – The pseudo-random outputs of the hash functions are combined to produce the initial noise for the diffusion denoising process. 
Taken together, these steps set an initial noise that is both distortion-free with respect to standard random initialization and correlated with the semantics of the input prompt (see \Cref{sec:analysis}). We describe our watermarking method in detail below.

\begin{figure}[t]
  \centering
  \begin{subcaptionbox}{\label{fig:cat_image}}[0.45\linewidth][c]{%
    \includegraphics[height=4cm]{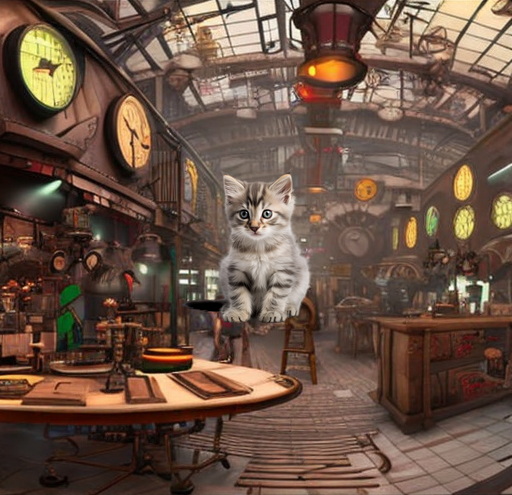}}
  \end{subcaptionbox}
  \hfill
  \begin{subcaptionbox}{\label{fig:cat_patches}}[0.45\linewidth][c]{%
    \includegraphics[height=4cm]{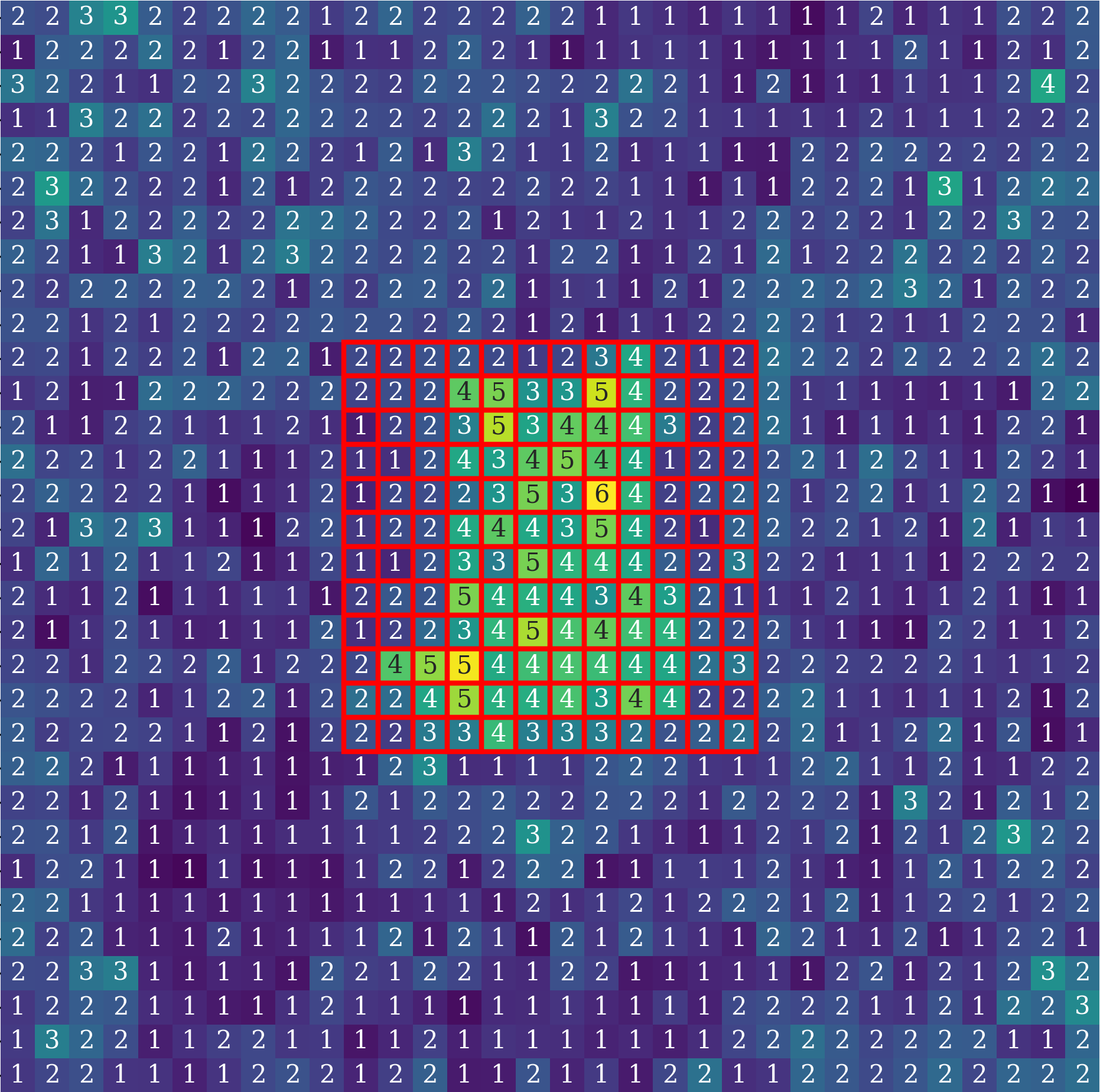}}
  \end{subcaptionbox}
    \caption{\textit{\textbf{Effect of the Cat Attack on \methodname.} (Left) A cat image is pasted onto a watermarked image at a random position and scale. (Right) Our method detects this modification by identifying elevated $\ell_2$ norms in affected patches. Note that the displayed norms are rounded to the nearest integer.}}
  \label{fig:cats}
\end{figure}

\subsection{Method}

Formally, our method first creates a semantic vector $\mathbf{v}$ and uses it to sample the initial noise $\mathbf{z}$ for the watermarked image. During detection, we aim to verify the connection between the used initial noise $\mathbf{z}$ and the semantic embedding of the image.
When approximating $\mathbf{z}$ from the generated image during detection and verifying it, we consider the following error sources: 
\begin{itemize}
    \item We do not have access to $\mathbf{v}$ at detection time; instead, we must use an approximate version $\tilde{\mathbf{v}}$ derived from the image we are analyzing.
    Using $\tilde{\mathbf{v}}$, we produce an approximate version of the used initial noise $\tilde{\mathbf{z}}$.
    \item Because of the randomness in the diffusion process and its inversion, we cannot estimate $\mathbf{z}$ accurately; instead, we get an approximation of the inverted noise $\mathbf{z}^\text{inv}$.
\end{itemize}

Ideally, a watemarked image would yield a perfect match between the noisese derived from the image semantics $\mathbf{z}^{\text{inv}}$ and the inverted $\tilde{\mathbf{z}}$ but this is not guaranteed because both differ from $\mathbf{z}$ due to the error sources mentioned above. 
Yet, we can mitigate this by independently embedding the image semantics across multiple patches. Therefore, our method provides a high likelihood that even if some patches do not match because of the challenges discussed above, many of the patches will match as long as the suspect image is watermarked.

\begin{algorithm}
\caption{Watermark Generation}
\label{alg:watermark_generation}
\begin{algorithmic}[1]
\STATE \textbf{Input:} \texttt{prompt}: text prompt, $n$: number of patches, $b$: number of bits per patch, \texttt{salt}: secret salt
\STATE \textbf{Output: } Watermarked image of \texttt{prompt}
\STATE $\mathbf{z}^{\textnormal{pre}} \sim \mathcal{N}(\mathbf{0}, \mathbf{I})$
\STATE $\mathbf{x}^{\textnormal{pre}} \gets \textnormal{Diffusion}(\mathbf{z}^{\textnormal{pre}}, \texttt{prompt})$
\STATE $\mathbf{v} \gets \textnormal{Embed}(\textnormal{Caption}(\mathbf{x}^\textnormal{pre}))$
\FOR{$i=1,\ldots, n$}
    \STATE $\mathbf{z}_i \gets \texttt{SimHash}(\mathbf{v}, i, \texttt{salt})$
\ENDFOR
\STATE \textbf{return} $\text{Diffusion}(\mathbf{z}, \texttt{prompt})$
\end{algorithmic}
\end{algorithm}

\begin{algorithm}
\caption{\texttt{SimHash}}
\label{alg:simhash}
\begin{algorithmic}[1]
\STATE \textbf{Input:} $\mathbf{v}$: semantic vector, $i$: patch index, \texttt{salt}: secret salt, $b$: number of bits, \texttt{hash}: cryptographic hash function
\STATE \textbf{Output:} Semantic, secure, normally distributed noise
\STATE $\texttt{bits} \gets \mathbf{0}$ \hfill // Initialize hash input 
\FOR{$j = 1, \dots, b$}
    \STATE // Reproducibly sample random vector
    \STATE $s \gets \texttt{hash}(i,j,\texttt{salt})$
    \STATE Sample $\mathbf{r}_j^{(i)} \gensimvec \mathcal{N}(\mathbf{0}, \mathbf{I})$
    \STATE $\texttt{bits}[j] \gets \text{sign}(\langle \mathbf{v}, \mathbf{r}_j^{(i)} \rangle)$ \hfill // Random projection
\ENDFOR
\STATE $s_i \gets \texttt{hash}(\texttt{bits}, i, \texttt{salt})$
\STATE \textbf{return} $\mathbf{z}_i \gensimi \mathcal{N}(\mathbf{0}, \mathbf{I})$
\end{algorithmic}
\end{algorithm}

\subsection*{Watermark Generation}

The first step of the generation process is to find a semantic vector $\mathbf{v}$ describing the image that will be generated. Ideally, the semantic vector depends only on the prompt and correlates exclusively with images generated from it. Yet, in practice, predicting the final image semantics based on the user prompt is difficult.

To approximate the generated image semantics, our solution begins by generating a proxy image $\mathbf{x}^{\textnormal{pre}}$.
We first caption the image using
BLIP-2 model \cite{li2023blip}.
Then, the caption is embedded into a latent semantic space using the Paraphrase Mpnet Base V2 model \cite{reimers-2019-sentence-bert}, resulting in a semantic vector $\mathbf{v}$ which captures the high-level semantics of the generated image by the prompt (a similar concept also explored in the concurrent work of SWIFT~\cite{evennou2024swift}).

\textbf{Semantic Embedding Optimization.}
During detection, the generated image will be captioned to obtain a semantic vector $\Tilde{\mathbf{v}}$, approximating the semantic vector $\mathbf{v}$ that was used to seed the random noise. To ensure a similarity between $\mathbf{v}$ and $\Tilde{\mathbf{v}}$, it is not enough only to generate the proxy image $\mathbf{x}^{\textnormal{pre}}$ with the same prompt (a qualitative comparison of the captioned proxy image and the final generated image can be found at~\Cref{fig:proxy_vs_actual}). Therefore, to encourage the embedding $\mathbf{v}$ to correlate to $\Tilde{\mathbf{v}}$ and not to unrelated vectors, we fine-tuned the embedding model using 10k pairs of related captions, leading to additional improvements (\Cref{fig:embedding-ablation}). The full implementation details of the fine-tuning process can be found in~\Cref{app:embedding_opt}.

After obtaining the desired semantic embedding, we generate the watermarked noise $\mathbf{z}$ using the semantic vector ${\mathbf{v}}$ and the SimHash algorithm described below.
Finally, we will use the diffusion mode to generate the image with the watermarked initial noise. The generation algorithm is summarized in \Cref{alg:watermark_generation}.

\subsection*{Semantic Patterns with SimHash}

The core subroutine of our watermarking method is SimHash \cite{charikar2002similarity}, used to generate initial noise patches correlated to a given vector (Algorithm \ref{alg:simhash}).
SimHash takes a vector \( \mathbf{v} \) and generates an initial noise \( \mathbf{z}_i \) for patch $i$, allowing a verifier to later determine whether \( \mathbf{z}_i\) is related to \( \mathbf{v} \). Namely, the semantic vector $\mathbf{v}$ is passed through a locality-sensitive hashing method that generates representations of $\mathbf{v}$ in terms of its projections onto random directions.

Specifically, SimHash projects $\mathbf{v}$ onto a set of random vectors for each patch of the initial noise map. It uses $b$ projection vectors for each of the $k$ noise patches.
Each noise patch is generated using a seed determined by the sign of the projection of the semantic vector onto each of the $b$ projection directions. 
For $i \in \{1,\ldots, k\}$, the seed and the noise for patch $i$ are:
\begin{align}
    s_i &= \texttt{hash}(\sign(\langle\mathbf{v}, \mathbf{r}_1^i\rangle),\ldots,\sign(\langle\mathbf{v}, \mathbf{r}_b^i\rangle), i, \texttt{salt}).
    \nonumber \\
    \mathbf{z}_i &\gensimi \mathcal{N}(\mathbf{0}, \mathbf{I}) \nonumber
\end{align}
This ensures that similar semantic vectors would yield similar hash values.

Yet, having repetitive bit inputs ($s_i$) may result in repetitive patches in the initial noise, and therefre may distort image generation. Therefore, we include the patch index in the hash function input to ensure that $s_i \neq s_j$ even when the input bits are identical (see \Cref{fig:carpet} for generation samples in the case of repetitive noise patches).
For cryptographic security, we also use a user-specific secret salt.

\begin{algorithm}
\caption{Watermark Detection}
\label{alg:watermark_detection}
\begin{algorithmic}[1]
\STATE \textbf{Input: } $\tilde{\mathbf{x}}$: suspect image, 
$\tau:$ patch distance threshold, $n$: number of patches, $m^\text{match}$: match threshold, $b$: number of bits per patch, \texttt{salt}: secret salt
\STATE \textbf{Output: } Watermark detection (\textbf{True}/\textbf{False})
\STATE $\tilde{\mathbf{v}} \gets \textnormal{Embed}(\textnormal{Caption}(\tilde{\mathbf{x}}))$
\STATE $\mathbf{z}^{\textnormal{inv}} \gets \text{InverseDiffusion}(\tilde{\mathbf{x}})$
\STATE $m \gets 0$
\FOR{$i=1,\dots,n$}
\STATE $\tilde{\mathbf{z}}_i \gets \texttt{SimHash}(\tilde{\mathbf{v}}, i, \texttt{salt})$
\IF{$ \|\tilde{\mathbf{z}}_i - \mathbf{z}^\text{inv}_i \|_2 < \tau$}
    \STATE $m ++$
\ENDIF
\ENDFOR
\STATE \textbf{return} $m \geq m^\text{match}$
\end{algorithmic}
\end{algorithm}

\subsection*{Watermark Detection}
For detection, we generate noise based on the semantic content of the image and check how well it corresponds to the reconstructed noise obtained through DDIM inversion (\Cref{alg:watermark_detection}).
We begin by embedding the image to get a semantic vector $\tilde{\mathbf{v}}$ that captures the content of the image.
SimHash is then applied to this vector as in the watermark generation process, generating an estimated initial noise $\tilde{\mathbf{z}}$.
Finally, we use inverse diffusion (e.g., DDIM \cite{song2020denoising}) to approximately reconstruct the initial noise $\mathbf{z}^\text{inv}$ from the image.

Since $\mathbf{v}$ and $\tilde{\mathbf{v}}$ may differ, the originally used noise $\mathbf{z}$ and $\tilde{\mathbf{z}}$ are not necessarily the same.
However, by the similarity property of SimHash, $\mathbf{z}$ and $\tilde{\mathbf{z}}$ will be identical on some patches with very high probability as long as $\mathbf{v}$ and $\tilde{\mathbf{v}}$ are close.
On any patch number $i$ where the SimHash patch match ($\tilde{\mathbf{z}}_i = \mathbf{z}_i$), we get:
\begin{align}
    \| \tilde{\mathbf{z}}_i - \mathbf{z}^\text{inv}_i\|_2
    = \| \mathbf{z}_i - \mathbf{z}^\text{inv}_i\|_2.
\end{align}
For such patches, the only challenge in identifying the watermark stems from the difference between the originally used noise and the reconstructed noise through DDIM.
Empirically, we find that the $\ell_2$-norm of the difference between the inverted and expected noise patches ($\| \tilde{\mathbf{z}}_i - \mathbf{z}^\text{inv}_i\|_2$) allows us to detect whether the inverted noise patches originated from the suspected noise patch with a $>99.9\%$ ROC-AUC.


\begin{figure}[t!]
    \centering
    \includegraphics[width=0.55\linewidth]{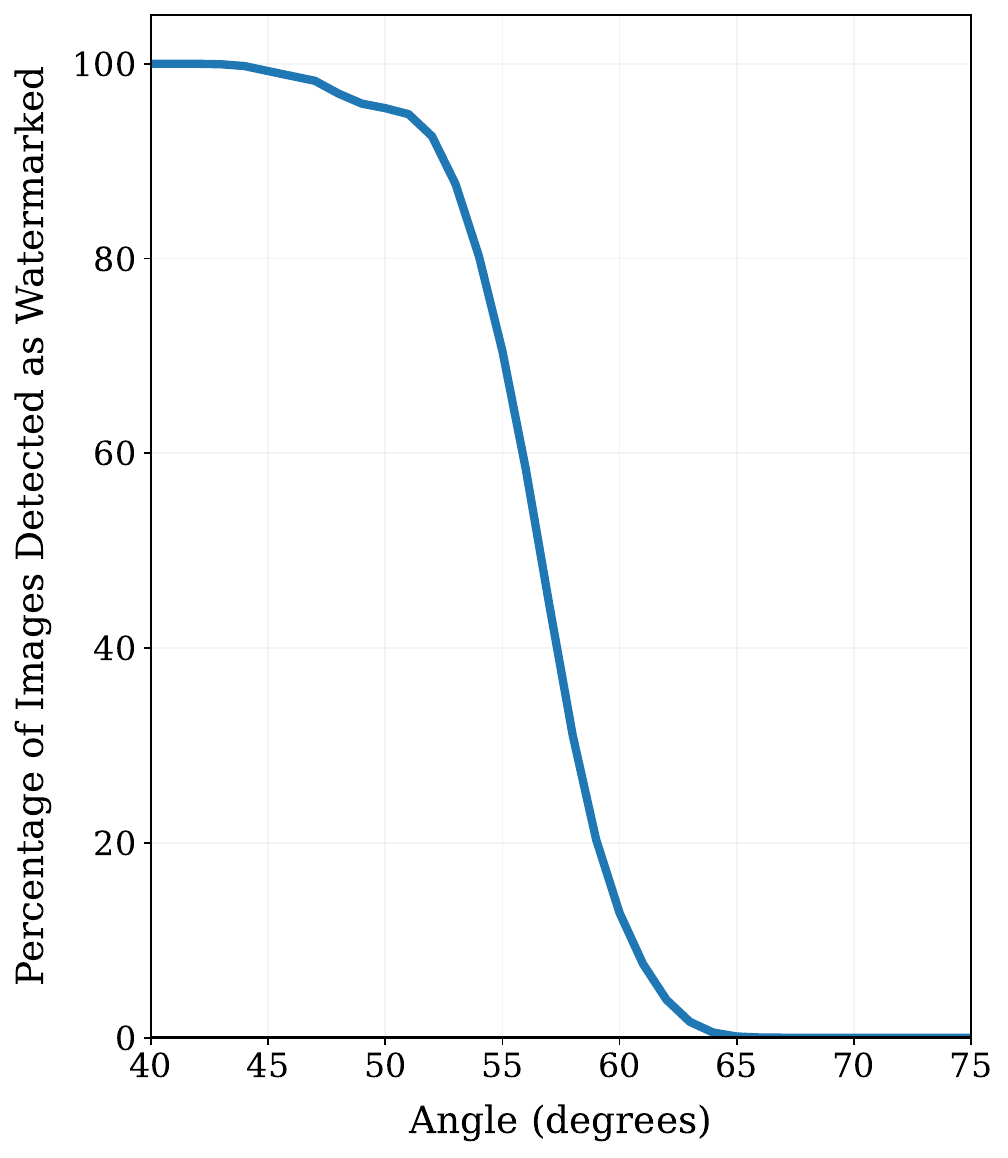}
    
    \vspace{1em}
    
    {\scriptsize
    \setlength{\tabcolsep}{8pt}
    \begin{tabular}{cc}
        \toprule
        \textbf{Angle $\theta(\mathbf{v}, \tilde{\mathbf{v}})$} & \textbf{Detection Probability} \\
        \midrule
        $65^\circ$ & $8.55 \times 10^{-4}$ \\
        $60^\circ$ & $0.053$ \\
        $55^\circ$ & $0.551$ \\
        $50^\circ$ & $0.998$ \\
        $45^\circ$ & $1.000$ \\
        \bottomrule
    \end{tabular}
    }
    \caption{\textit{\textbf{Watermark Detection vs. Semantic Similarity.}} We plot the empirical probability of detecting a watermark as a function of the angle between the semantic embedding used for watermark generation and that of the inspected image ($n=1024$, $b=7$). The table shows the analytical detection probabilities at key angles calculated by Lemma \ref{lemma:detect}, illustrating how sharply SimHash distinguishes semantically related images from unrelated ones.}
    \label{fig:sem_angle}
\end{figure}

\paragraph{Semantic Similarity Detection.}
Finally, in order to detect whether an image was initially generated with our watermark, we count the number of patches that \textit{match} (i.e., their $\ell_2$-norm distance is below a threshold $\tau$).
If the number of matches is above a set threshold $n^\text{match}$ then we declare the image is watermarked.
In \Cref{sec:analysis}, we analyze the probability of correctly identifying a watermarked image.

\paragraph{Tampering Detection with a Spatial Test.}\label{par:spatial_test_detection}
In addition to the association between the watermark and the semantic embedding, edits such as object addition, removal, or modification are likely to alter the estimated initial noise in the affected image regions. This enables our watermark to provide localized information about edits that might have been made to the image.
Consequently, even when the semantic embedding of the image $\tilde{\mathbf{v}}$ aligns well with the initial embedding used to seed the noise $\mathbf{v}$, such tampering edits can still be detected by identifying localized patches in the reconstructed initial noise that neither match the expected noise nor any other valid input to the hash function. 

To detect such cases, we may inspect the noise patches one by one. Given the model owner's private information, we may recover the $b$ input bits used to seed each patch with an exhaustive search over the $2^b$ options per patch, and recover a matching initial noise.
Comparing this reconstructed noise to the inverted noise $\mathbf{z}^{\text{inv}}$ allows us to detect which patches may have been modified. The total time for this search scales as $n \cdot 2^b$ (which is much faster than naively searching over all $2^{(b \cdot n)}$ possible initial noise). After obtaining a per patch noise-matching map (see \Cref{fig:cat_patches}), we may apply a \textit{spatial test} as the one described in \Cref{app:spatial_test} to detect tampering attempts.
In any case, the local patch inspection is only required when an image is deemed watermarked by semantic similarity detection; but the watermark owner would like to have a finer understanding of the edits that might have been applied to it. This inspection is especially useful against the \attackname, described in \Cref{sec:experiments}.

\begin{figure*}[t!]
    \centering
    \begin{subfigure}{0.32\textwidth}
        \centering
        \includegraphics[width=\linewidth]{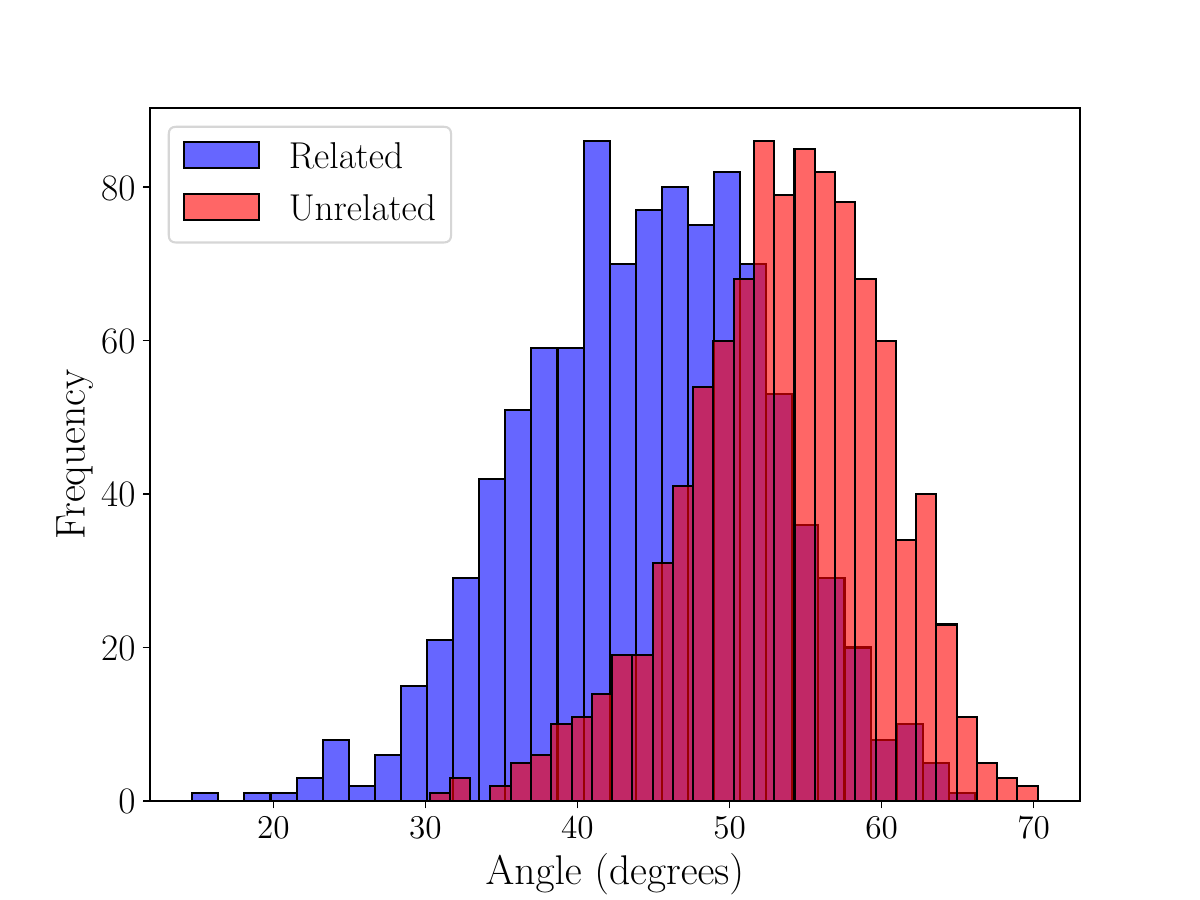}
        \caption{\textit{\textbf{Image Feature Vector.}} Direct use of image feature embedding fails to separate related from unrelated images.}
        \label{fig:ablation-clip}
    \end{subfigure}
    \hfill
    \begin{subfigure}{0.32\textwidth}
        \centering
        \includegraphics[width=\linewidth]{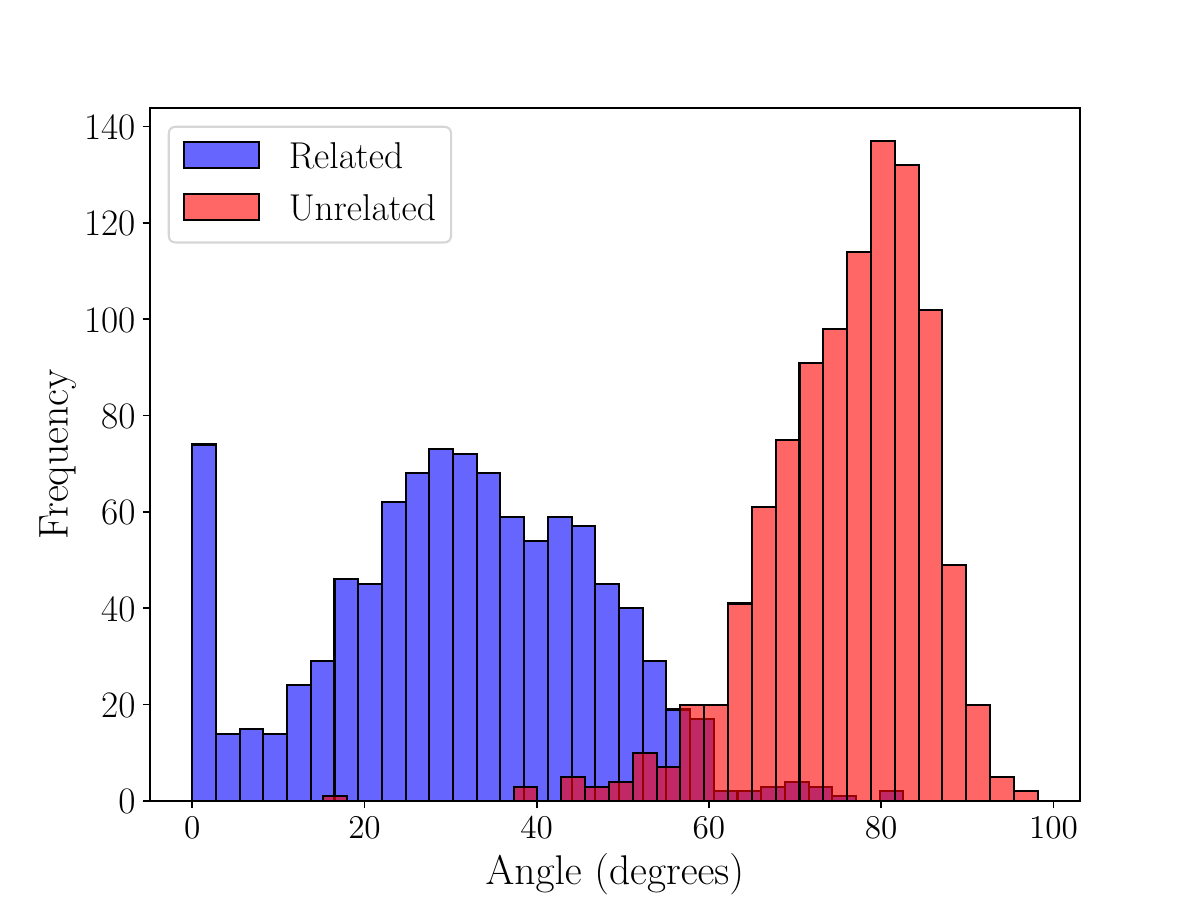}
        \caption{\textit{\textbf{Caption Embeddings.}} Employing caption embeddings from \textit{blip2-flan-t5-xl} and \textit{paraphrase-mpnet-base-v2} yields improved separation.}
        \label{fig:ablation-paraphrase}
    \end{subfigure}
    \hfill
    \begin{subfigure}{0.32\textwidth}
        \centering
        \includegraphics[width=\linewidth]{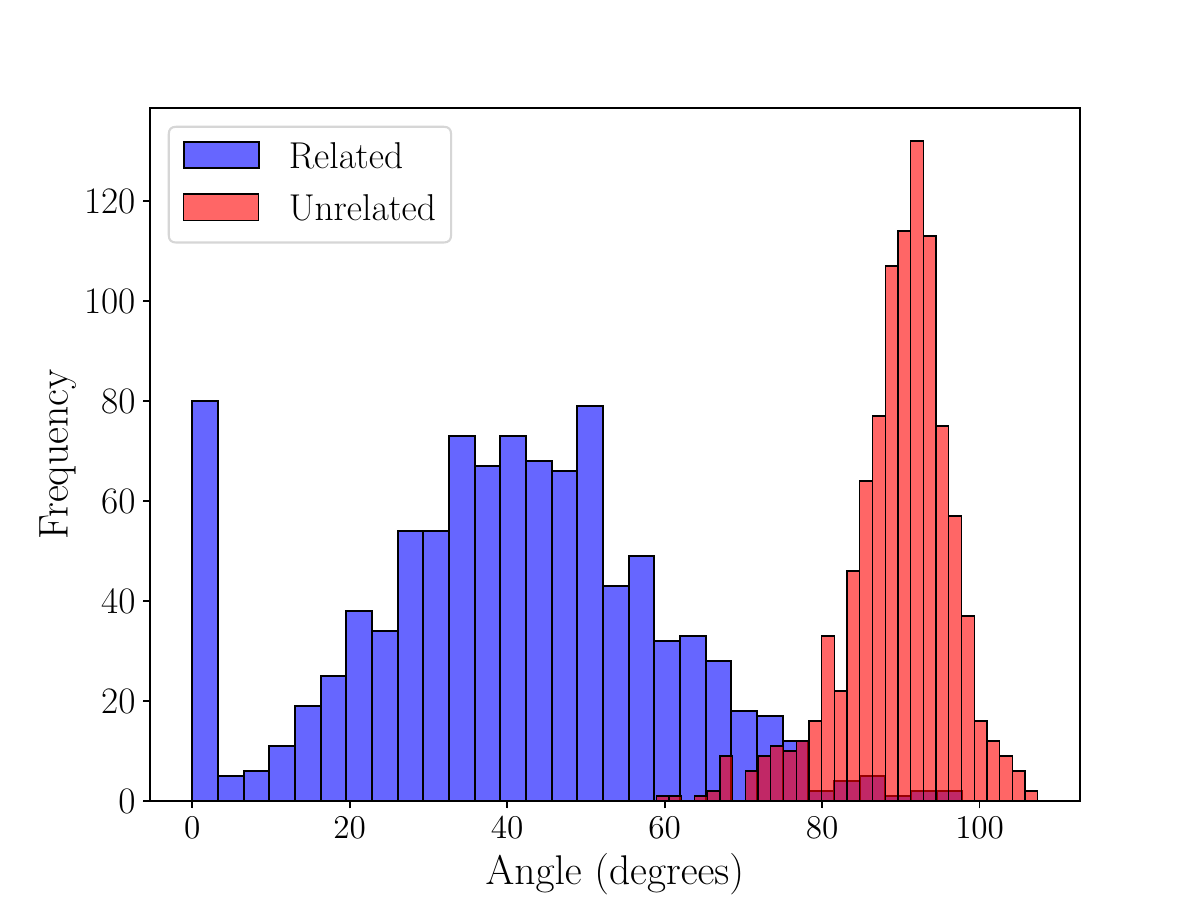}
        \caption{\textit{\textbf{Fine-tuned Caption Embeddings.}} Fine-tuning the embedding model on 10k caption pairs further enhances separation.}
        \label{fig:ablation-finetuned}
    \end{subfigure}
    \caption{\textit{\textbf{Ablation of Embedding Strategies for Watermark Detection.}} Comparison of angle separation between related and unrelated images using different embedding approaches. The raw image feature vector (left) fails to distinguish semantic relationships, while caption embeddings (center) substantially improve separation. Fine-tuning the embedding model (right) yields additional gains in detection accuracy.}
    \label{fig:embedding-ablation}
\end{figure*}

\subsection{Analysis}
\label{sec:analysis}

Before formally analyzing our watermarking scheme, we state a simplifying assumption on the distance between the initial and reconstructed noise patches.
We assume the noise patches are close if and only if the suspect image was produced from the same noise as the one given by our watermarking scheme.
The impact of low-likelihood events, where unrelated patches end up close after noise reconstruction, remains part of our empirical analysis in \Cref{sec:experiments}.

\begin{assumption}[Patch 
\label{asumption:ideal}
Distance Separation]\label{ass:patch}
    There is a threshold $\tau^{\textnormal{dist}}$ so that,
    for all generation noises $\mathbf{z}$, inverted noises $\mathbf{z}^{\textnormal{inv}}$, and patches indeces $i \in [k]$,
    \begin{align*}
       \| \mathbf{z}_i - \mathbf{z}_i^{\textnormal{inv}} \|_2 \leq \tau^{\textnormal{dist}}
    \end{align*}
    if and only if $\mathbf{z}^\textnormal{inv} = \textnormal{InverseDiffusion}(\textnormal{Diffusion}(\mathbf{z}))$.
\end{assumption}

An immediate consequence of the patch distance separation assumption is that we never declare an image as watermarked if its initial noise was not generated using our watermarking scheme.

\paragraph{Unrelated Prompts.} A key property of our watermarking approach is its resistance to forgeries generated from unrelated prompts. Prior watermarking methods declare an image as watermarked as long as the watermarking pattern is embedded in the initial noise and the diffusion and inverse diffusion processes remain reasonably accurate. However, this creates vulnerabilities - an adversary could take an existing watermark and apply it to an unrelated, potentially offensive, or misleading prompt. In contrast, our approach strengthens watermark integrity by requiring that the new prompt remains semantically close to the original. This ensures that watermarks are not erroneously detected in entirely unrelated images. We formalize this claim the the lemma below. 

\begin{lemma}[Detection Probability]\label{lemma:detect}
    Consider a suspect image $\tilde{\mathbf{x}}$ produced from our watermarking scheme with initial semantic vector $\mathbf{v}$.
    Let $\tilde{\mathbf{v}}$ be the (possibly quite different) semantic embedding of $\tilde{\mathbf{x}}$, and $\theta \in [-90\degree, 90\degree]$ be the angle between $\mathbf{v}$ and $\tilde{\mathbf{v}}$.
    Set $\theta^\text{mid}$ as the threshold between semantic vectors we deem \textnormal{related} vs. \textnormal{unrelated}.
    The probability that we identify the image as watermarked is
    \begin{align}\label{eq:match_prob}
        \sum_{k=\lfloor n \rho(\theta^\text{mid})\rfloor}^{n}
        \binom{n}{k} \rho(\theta)^k (1-\rho(\theta))^{n - k}.
    \end{align}
    where $\rho(\theta) =\left( 1- \frac{\theta}{180\degree}\right)^b$.
\end{lemma}

We illustrate in the example below the sharp detection thresholds Lemma \ref{lemma:detect} implies. Namely, we show how the watermark detection probability varies with semantic similarity between the original and a potentially modified image. 
We delay the proof of Lemma \ref{lemma:detect} to the appendix.


\begin{example}[Sharp Detection Thresholds] 
    Our watermarking scheme embeds a semantic vector $\mathbf{v}$ into an image at generation time.
    When evaluating a suspect image that was generated via our watermark, we extract its current semantic vector $\tilde{\mathbf{v}}$. The probability of a watermark detection depends on the semantic angle $\theta(\mathbf{v}, \tilde{\mathbf{v}})$ between $\mathbf{v}$ and $\tilde{\mathbf{v}}$.
    
    For instance,~\Cref{fig:ablation-finetuned} illustrates a separation between vectors associated with the original image and those that are unrelated, occurring at a threshold of approximately $\theta^\text{mid} \approx 55\degree$.
    When our watermarking scheme is run with $\theta^\text{mid} = 55\degree$, $n=1024$, and $b=7$ (see~\Cref{fig:kb-ablation} for ablation of these parameters).~\Cref{fig:sem_angle} quantifies the probability of a watermark detection:
    We observe near-perfect separation between related and unrelated watermarked images for angles exceeding $5\degree$ beyond the threshold.
    For comparison,~\Cref{app:bin_switch} presents the semantic angle shift resulting from the simple insertion of different objects.

\end{example}

\section{Empirical Analysis}
\label{sec:experiments}

In this section, we empirically evaluate the robustness of \methodname to different attacks. 

\paragraph{Setting.} To ensure a fair comparison with prior work~\cite{wen2023tree,ci2024ringid,arabi2024hidden}, we use Stable Diffusion-v2~\cite{rombach2022high} with 50 inference steps for both generation and inversion for all methods. Evaluations were conducted on a set of prompts sourced from~\cite{gustavosta2024pro}. We set \( n = 1024 \) and \( b = 7 \) for all experiments. An ablation study on the effects of $n$ and $b$ is available in \Cref{sec:abiliation}.

\paragraph{Regeneration with the Private Model.}
Prior works assume that the attacker lacks access to the model weights (which are needed for accurate DDIM noise inversion) and that the noise used during generation cannot be forged or approximated with sufficient accuracy~\cite{arabi2024hidden}. 
Going beyond previous studies, we consider here a more challenging scenario in which the attacker has full access to the model weights and can invert the generated image using the same model that produced it. The attacker's access to the private model is taken as an upper bound for the attacker's capability in practical forgery attacks \cite{arabi2024hidden,muller2025black,jain2025forging}.

In our experiment, we first generate an image using watermarked noise. We then perform an inversion with the same model to recover an approximate initial noise, which is subsequently used to generate a second, forged, image. Because the attack prompt differs from the original prompt, the semantic embedding of the image $\Tilde{\mathbf{v}}$   changes to ${\mathbf{v}}_{attack}$. The detection algorithm, therefore compares the estimated noise to a reference derived from ${\mathbf{v}}_{attack}$ (and not from $\Tilde{\mathbf{v}}$).
The noise pattern derived from ${\mathbf{v}}_{attack}$ during detection is less likely to correlate to the pattern embedded in the image, enabling the detection algorithm to declare the image as not watermarked and evade the attack.
As can be seen in \Cref{tab:regenneration}, our method uniquely provides non-trivial robustness in this setting.

We also evaluate the Latent Forgery Attack directly \cite{jain2025forging} in \Cref{sec:latent_forgery_attacks}.

\paragraph{Cat Attack.}
A significant practical threat to the reputation of a watermark owner arises from localized modifications that shift the semantic interpretation of a watermarked image, as opposed to producing a wholly new image. To evaluate our method's resilience against such tampering, we introduce an evaluation we term the \attackname{}.

In this experiment, a cropped object (e.g., a cat) is pasted onto a watermarked image. The cat image is randomly resized to between 30\% and 60\% of the watermarked image’s dimensions and placed at a random location, as exemplified in \Cref{fig:cat_image}. Unlike previous watermarking techniques that may overlook semantic content, our approach is designed to detect such alterations. 

As shown in \Cref{fig:cat_patches}, the pasting of the object leads to elevated $\ell_2$ norms in the affected patches. Quantitative results are presented in \Cref{tab:cat-attack}, and comparison of different object sizes can be found in the Table in \Cref{tab:cat_attack_app}. These results reveal that while our basic detection offers some robustness, integrating the local spatial test (described in \Cref{sec:method}) significantly improves the detection of these edits. This demonstrates a key advantage of our method.
We note that in a different setting, methods such as \cite{zhang2024editguard,sander2024watermark} offer a strong solution for tampering detection via post-hoc watermarking. 
Since robustness to tampering can be in tension with resistance to removal attacks, we next analyze our method's performance against standard removal attacks.

\paragraph{Regeneration Based Removal Attack.} Our method is robust to regeneration-based removal attacks~\cite{zhao2025invisible}, similarly to other initial-noise-based approaches~\cite{gunn2024undetectable,arabi2024hidden,yang2024gaussian}, and it significantly outperforms classical watermarking methods (see~\Cref{app:regenration_attacks}).

\paragraph{Steganalysis Removal Attack.} We evaluate the robustness of our method against a steganalysis attack~\cite{yang2024steganalysis} that attempts to approximate the watermark by subtracting non-watermarked images from watermarked ones. As shown in~\Cref{tab:steganalysis}, \methodname~ maintains high performance under this attack.

\begin{table}[t]
    \centering
    \caption{\textit{\textbf{Detection of the Cat Attack.} ROC-AUC of detecting edits in generated images, as described in \Cref{sec:experiments}.}}  
    \begin{tabular}{lc}
        \toprule
        \textbf{Method} & \textbf{AUC} \\
        \midrule
        WIND & 0.000 \\
        Tree-Ring & 0.000 \\
        Gaussian Shading & 0.000 \\
        \methodname & 0.551 \\
        \methodname + Spatial Test & \textbf{0.982} \\
        \bottomrule
    \end{tabular}
    \label{tab:cat-attack}
\end{table}

\begin{table}
    \centering
    \caption{\textit{\textbf{Robustness to Private Model-Based  Forgery Attack.} An attacker with access to the private model weights can approximate the watermarked initial noise by inverting a watermarked image using the private model. We evaluate how accurately different methods evade the false identification of unrelated images, generated with this initial noise, as watermarked.}}
    \begin{tabular}{lc}
        \toprule
        \textbf{Method} & \textbf{AUC} \\
        \midrule
        WIND & 0.000 \\
        Tree-Ring & 0.000 \\
        Gaussian Shading & 0.000 \\
        \methodname & \textbf{0.708} \\
        \bottomrule
    \end{tabular}
    \label{tab:regenneration}
\end{table}

\paragraph{Robustness to Image Transformations.}
We evaluated the robustness of \methodname~ under a standard suite of image transformations (see \Cref{app:transformation}). As shown in~\Cref{fig:image-transformations}, \methodname~achieves an average ROC-AUC of $0.896$ under these conditions, computed via the Spatial Test detector (\Cref{par:spatial_test_detection}). This is comparable to some watermarking techniques and somewhat lower than others \cite{wen2023tree}. Yet, our method provides a unique resistance to forgery. Further enhancements, such as incorporating rotation search or sliding-window search during detection (see \cite{arabi2024hidden}), could improve its robustness against removal attempts.


\paragraph{Ablation of Captioning and Embedding Models.}
A straightforward approach for our method to approximate the final image semantics to embed it in the noise would be to use the visual feature vector from the proxy-generated  $\mathbf{x}^{\textnormal{pre}}$ rather than the embedding of its caption. However, as illustrated in \Cref{fig:ablation-clip}, this approach fails to yield a clear separation between related and unrelated images. Consequently, we employ the captioning and caption-embedding for deriving caption embeddings, which results in a more distinct separation as shown in \Cref{fig:ablation-paraphrase}. To further enhance our method's accuracy, we fine-tuned the embedding model using 10k pairs of related captions, leading to additional improvements (\Cref{fig:ablation-finetuned}).

\paragraph{Generation Quality.}
Our watermarking method is distortion-free at the single-image level, since the added noise is sampled from a pseudo-random Gaussian distribution, similarly to non-watermarked image generation. As a result, all single-image quality metrics remain identical to those of non‑watermarked images (see~\Cref{tab:clip_results}, \Cref{fig:generated_images}). 

\section{Limitation and Discussion}
\label{sec:discussion}

\paragraph{Stronger Forgery Attacks.}
Although we evaluated a stronger set of forgery attacks compared to previous works, other types of forgery attacks might still potentially compromise our watermark. For example, a highly persistent attacker might attempt to gather information about the correlation between individual initial noise patches and the image semantics. While not theoretically impossible, an attacker would face several practical limitations in carrying out such an attack. Among them are the lack of access to the private model weights, the inherent stochasticity of the watermark, and the watermark owner's ability to deploy multiple instances of the hash function by using multiple secret salts.

\paragraph{Attacker Advantage and Removal Attacks.} Our method is more vulnerable to removal attacks than some existing methods. However, we believe that a sufficiently persistent attacker can remove most current watermarks. 
Nonetheless, improving watermark robustness against forgery attacks holds significant societal value - it is essential for protecting the model owner's reputation and, consequently, for enabling practical deployment.

\noindent\paragraph{}Additional limitations and discussion points can be found in \Cref{app:additional_limit}.

\section{Conclusion}
\label{sec:conclusion}
We introduce the first initial noise-based watermarking method for diffusion models that is both database-free and semantic-aware. Our suggested watermark is uniquely robust against a new class of stronger forgery attacks. We hope our work highlights the potential of semantic-aware watermarking and helps pave the way forward for further research in this area.

\section*{Acknowledgments}
We thank Jonas Thietke and Haidar Shreif from Ruhr University Bochum for their careful reading of our manuscript and for pointing out a clarity issue with Figure~6 in an earlier version of this paper. We also thank Shingo Kodama from Middlebury College for his help in validating these results.


\clearpage

{
    \small
    \bibliographystyle{ieeenat_fullname}
    \bibliography{main}
}

\clearpage
\setcounter{page}{1}
\maketitlesupplementary



\begin{figure}[t!]
    \centering
    \includegraphics[width=\columnwidth]{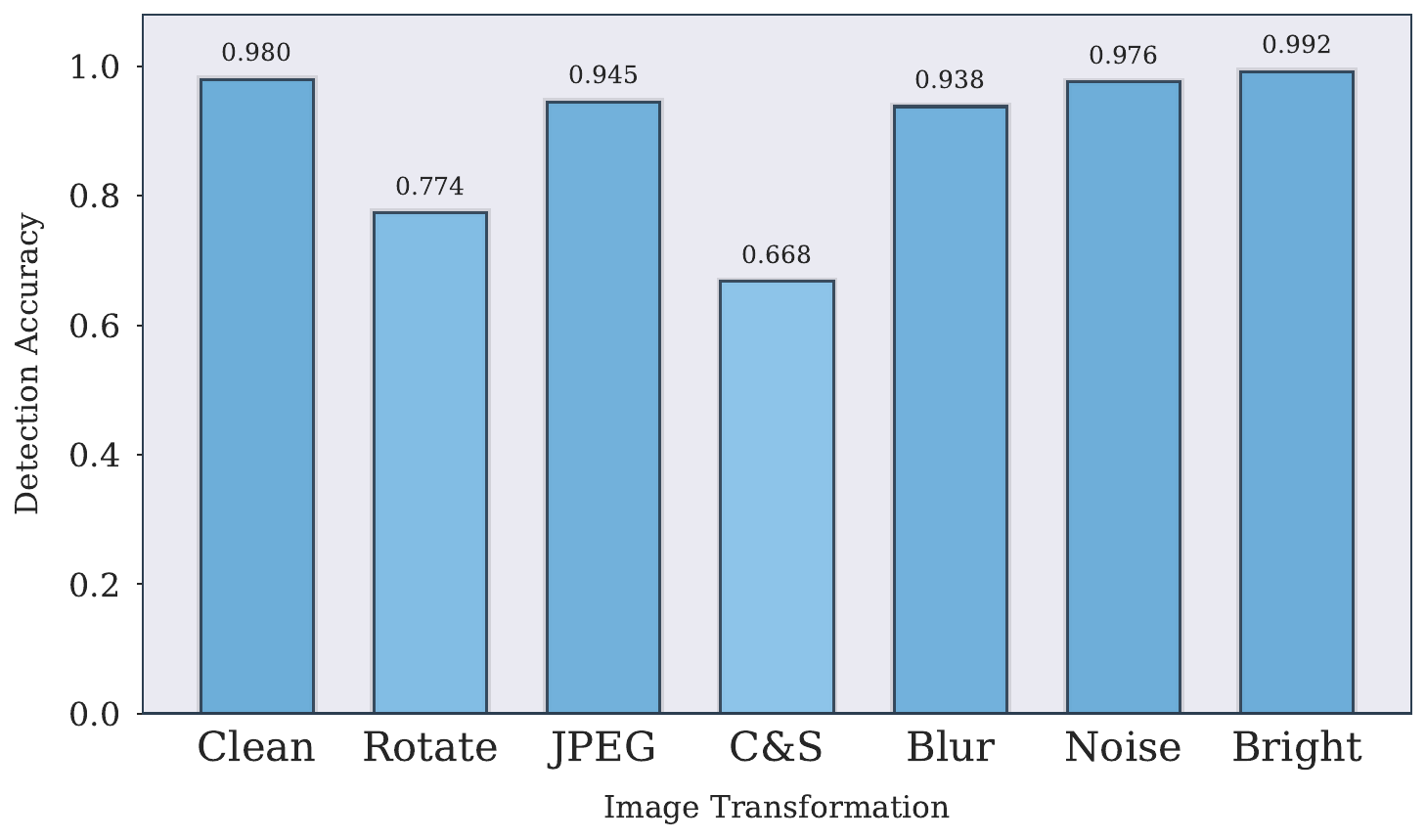}   
    \caption{\textit{\textbf{Robustness of Watermark Detection Against Image Transformations.} ROC-AUC of \methodname~under various image transformations, computed using the Spatial Test detector (\Cref{par:spatial_test_detection}).}}
    \label{fig:image-transformations}
\end{figure}

\begin{table*}[h!]
\centering
\caption{\textit{\textbf{Robustness of Steganalysis-Based Removal.} Comparison of performance metrics (ROC-AUC) under various levels of averaging.}}
\label{tab:steganalysis}
\begin{tabular}{lcccccccccc}
\toprule
Method  & 5 & 10 & 20 & 50 & 100 & 200 & 500 & 1000 & 2000 & 5000 \\
\midrule
Tree-Ring (AUC) & 0.293 & 0.267 & 0.314 & 0.275 & 0.214 & 0.228 & 0.211 & 0.224 & 0.224 & 0.241 \\
WIND (AUC) &  \textbf{1.000}
 & \textbf{1.000} & \textbf{1.000} & \textbf{1.000} & \textbf{1.000} & \textbf{1.000} & \textbf{1.000} & \textbf{1.000} & \textbf{1.000} & \textbf{1.000} \\
\methodname\ (AUC) & \textbf{1.000}
 & \textbf{1.000} & \textbf{1.000} & \textbf{1.000} & \textbf{1.000} & \textbf{1.000} & \textbf{1.000} & \textbf{1.000} & \textbf{1.000} & \textbf{1.000} \\
\bottomrule
\end{tabular}
\end{table*}

\section{Additional Related Works}
\label{app:related_works}
\paragraph{Post-Processing Methods.}

Post-processing techniques embed watermarks after the image generation stage, providing model-agnostic flexibility at the cost of potential quality degradation. Frequency-domain methods, such as methods using the Discrete Wavelet Transform (DWT) and Discrete Cosine Transform (DCT) \citep{navas2008dwt, al2007combined}, embed watermarks in the transformed domains and offer robustness against operations like resizing and translation. Complementing these, deep encoder-decoder frameworks such as HiDDeN \cite{zhu2018hidden} and StegaStamp \cite{tancik2020stegastamp} utilize end-to-end neural training for watermark embedding and extraction. Despite these advancements, however, these methods are vulnerable to regeneration attacks \cite{zhao2025invisible}. 
Alternative strategies operating in latent spaces have also been proposed \cite{fernandez2022watermarking}, though they also remain susceptible to sophisticated removal attacks. 

\section{Proof of Lemma \ref{lemma:detect}}

\begin{proof}[Proof of Lemma \ref{lemma:detect}]
The angle between the original semantic vector $\mathbf{v}$ used to generate the watermark and extracted semantic vector $\tilde{\mathbf{v}}$ of the suspect image is
\begin{align}
\nonumber
\theta(\mathbf{v}, \tilde{\mathbf{v}}) = \cos^{-1}\left(\frac{\langle \mathbf{v}, \tilde{\mathbf{v}}\rangle}{\| \mathbf{v}\|_2 \| \tilde{\mathbf{v}}\|_2}\right) \in [-90\degree, 90\degree].
\end{align}

By the property of SimHash\footnote{See Section 3 of \cite{charikar2002similarity} for details on why
\begin{align}
\label{eq:patch_angle}
\Pr_{\mathbf{r} \sim \mathcal{N}(\mathbf{0}, \mathbf{I})} (\sign(\langle\mathbf{v}, \mathbf{r}\rangle)~=~ \sign(\langle\tilde{\mathbf{v}},\mathbf{r}\rangle))~=~1-\frac{\theta}{180\degree}.
\end{align}} and Assumption \ref{ass:patch}, the probability\footnote{Technically, there is an additional chance of a random collision but, given the size of modern cryptographic hash functions like SHA2, we assume this probability is negligible.} that the $i$th patch aligns is
\begin{align}
    \nonumber
    \rho(\theta) := \Pr(\| \mathbf{z}_i - \tilde{\mathbf{z}}_i \|_2 \leq \tau)    
    = \left(1- \frac{\theta}{180\degree}\right)^b.
\end{align}
Since each SimHash instance is independent, the number of matches $m$ is distributed like a Binomial with $n$ trials and success probability $\rho(\theta)$.
During watermark detection, we count the number of patches that match, declaring an image watermarked if the number of matches exceeds the threshold $m^\text{match}$.
Setting $m^\text{match} = \lfloor n \rho(\theta^\text{mid})\rfloor$ yields the lemma statement.
\end{proof}

\section{Implementation Details}

\subsection{Key Parameters} Unless otherwise stated, the results are reported with the following parameters: number of patch matching threshold $n_{match} = 12$; patch-wise matching threshold $\tau = 2.3$; number of projection per noise patch: $b = 7$; number of noise patches $k = 1024$. All parameters were chosen to optimize the overall performance.

\subsection{Spatial Test}
\label{app:spatial_test}
To better analyze potential image tampering, we examine the structural organization of high-intensity regions in the patch correspondence heatmaps (see \Cref{sec:method}, Tampering Detection). Specifically, we threshold the heatmap data at the 80th percentile and identify connected components. The extracted parameter, the number of distinct clusters detected at this threshold, provides insight into the fragmentation of high-intensity regions. A higher number of clusters indicates a more dispersed distribution, while a lower number suggests more contiguous structures, which may be indicative of image tampering.

\subsection{Transformations for the Removal Attack}
\label{app:transformation}

We use a standard suit of transformations,  including a $75^\circ$ rotation, $25\%$ JPEG compression, $75\%$ random cropping and scaling (C \& S), Gaussian blur using an $8 \times 8$ filter, Gaussian noise with $\sigma = 0.1$, and color jitter with a brightness factor uniformly sampled between 0 and 6.

\subsection{Embedding Model Fine Tuning Process}
\label{app:embedding_opt}

\raggedright
\paragraph{Source Prompts and Caption Pairing.}
Prompts were sampled from MS-COCO and the Stable Diffusion Prompt Dataset (\texttt{Gustavosta/Stable-Diffusion-Prompts}). Images were generated using Stable Diffusion v2.1 (\texttt{stabilityai/stable-diffusion-2-1-base}), then captioned with BLIP-2 (\texttt{Salesforce/blip2-flan-t5-xl}). A regeneration loop used each caption as a prompt to generate a second image, which was again captioned, yielding semantically aligned $(\text{caption}_1, \text{caption}_2)$ pairs. Unrelated pairs were constructed by randomly mismatching captions. This procedure yielded 10{,}000 caption pairs.

\paragraph{Fine-Tuning.}
A SentenceTransformer model (\texttt{paraphrase-Mpnet-base-v2}) was fine-tuned using \texttt{MultipleNegativesRankingLoss}. Training was performed for 140 epochs with a batch size of 64 and 10\% warmup steps, using the AdamW optimizer with default settings.

We provide the full implementation details of the fine-tuning process to support reproducibility and enable further research.\footnote{Our fine-tuned caption embedding model is publicly available at \url{https://huggingface.co/kasraarabi/finetuned-caption-embedding}}.

\begin{figure}[t!]
   \centering
       \includegraphics[width=\linewidth]{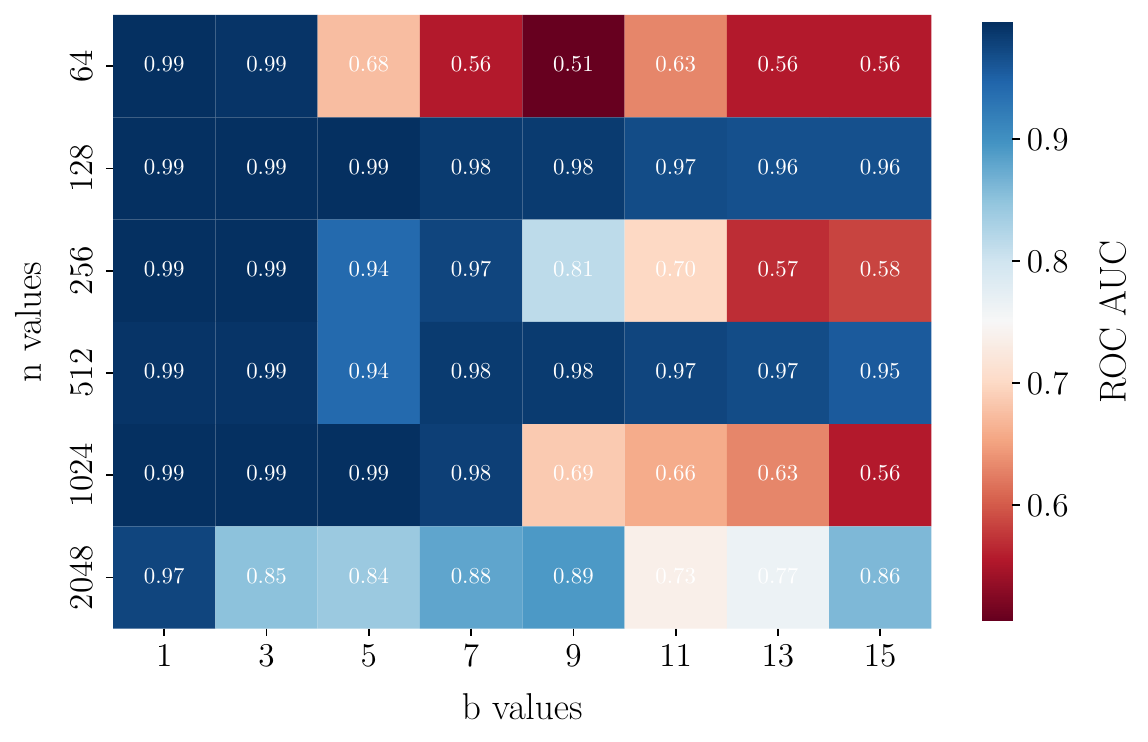}
   \caption{\textbf{\textit{Ablation Study of the Number of Patches ($\boldsymbol{n}$) and Bits ($\boldsymbol{b}$) on Watermark Detection Performance.}}}
   \label{fig:kb-ablation}
\end{figure}

\section{Ablation of Number of Patches and Bits}
\label{sec:abiliation}
To investigate the impact of the number of patches ($n$) and the number of bits ($b$) used to generate the initial noise, we conducted an exhaustive ablation study across various parameter combinations. The results are presented in \Cref{fig:kb-ablation}.

\section{Resilience to Latent Forgery Attacks}
\label{sec:latent_forgery_attacks}
We evaluate SEAL under the Latent Forgery Attack~\cite{jain2025forging}. This attack aims to adversarially perturb non-watermarked images such that they appear watermarked by mimicking the latent representation of an originally watermarked image. This type of attacks assumes access to at least one watermarked image and attempts to shift unrelated images into the watermarked image latent region \cite{muller2025black}.

Our experiments, conducted on 100 images, demonstrated complete robustness against such attacks. Due to its semantic binding, our watermark is closely entangled with the high-level content of the original image (used by the attacker). Therefore, it is unlikely that the watermark can be transferred to images featuring unrelated content. Beyond this, the task of forging the latent patches is itself highly nontrivial. Yet, aiming to forge noise patches that better align semantically with the original watermarked image (used during the attack) might yield greater success. We leave this direction for future research.

\section{Additional Limitations and Discussion}
\label{app:additional_limit}
\paragraph{Distortion-Free Property for Sets of Images}
Our watermarking scheme securely generates the noise for each patch from a normal distribution, ensuring that each individual noise is distributed from a normal distribution. However, multiple watermarked images corresponding to related prompts may leak information about the noise i.e., the noise in some patches will match while the noise in other patches does not. This leakage arises from our design choice to make similar prompts produce similar watermarks, a feature that enhances consistency but comes at the cost of some information exposure.

In contrast, some prior works do not exhibit this property and instead maintain a stronger sense of distribution-free randomness. Ignoring cases where the exact same noise is reused, such as when multiple images are generated by the same user in \cite{yang2024gaussian}, these methods ensure that each image is independently sampled from a normal distribution. This fundamental difference highlights a trade-off between ensuring independent randomness and enabling structured watermark consistency across related prompts. A user concerned about the distortion-free property for sets may vary the secret salt for different generations. This will allow the user to enjoy the best of both worlds, At the cost of searching through possible salts that may have been used during detection time.

\paragraph{Further Possible Improvement.} We made an initial attempt to find a semantic vector that is both known before generation and recoverable from the generated image. Yet, we believe this is a promising direction for future research. Improved semantic embedding methods, as well as approaches that jointly optimize image generation and semantic descriptor generation, could enhance the correspondence between the embedded watermark and the image's semantics. Such advancements may enable much stricter bounds on detecting when a watermarked image has been tampered with and how.

\paragraph{Watermarking Without a Proxy Image.} To watermark an image directly with SimHash, we may embed the noise
via post hoc diffusion inpainting (see Section 4.3 of \cite{arabi2024hidden}), at the
expense of a modest quality decrease; SSIM = 0.768. Alternatively, one may optimize an embedding of the prompt to correlate well between a givan prompt and the caption of the resulting image (similarly to \Cref{fig:embedding-ablation}).

\paragraph{Reliance on DDIM inversion.} While diffusion models are not accurately mapped back to the initial noise used to generate them, our method is based on an empirical observation: patches with small enough $\ell_2$  differences are almost always generated from the same seed, suggesting consistent behavior under approximate inversion. The $\ell_2$ distance, as a metric used to determine whether a reconstructed noise patch matches the original, yields over $99.9\%$ ROC-AUC.

\paragraph{Prompt Inversion Attack.} An attacker may choose to use prompt inversion methods such as~\cite{mahajan2024prompting}. When an attacker approximates the prompt, our semantic safeguard may be less effective. Yet, forgery attacks often
aim to harm the owner’s reputation. Our method ensures that even if a forged image is produced, its semantics remain somewhat close to the original watermarked image.

\paragraph{Attacking the Confusion Band.} As shown in \Cref{fig:embedding-ablation}, at certain levels of semantic similarity, our method may confuse images that are related or unrelated to the embedded caption. However, a forger’s ability to exploit this ambiguity is inherently constrained by the semantics of the original watermarked images — They can only forge images that are sufficiently similar to those whose watermark they supposedly managed to replicate. This limitation could potentially be mitigated in the future through improved embedding models.

\paragraph{Watermarking Capacity.} The ability to encode a very large number of distinct watermarks, or to identify one user out of many millions possible watermark owners, may raise concerns. As previous work has shown, noise-based watermarking methods can support millions of different users by using different hash salts, while still ensuring that the key patterns remain distinguishable \cite{gunn2024undetectable, arabi2024hidden}.

\begin{figure}[h]
    \centering
    \includegraphics[width=0.7\columnwidth]{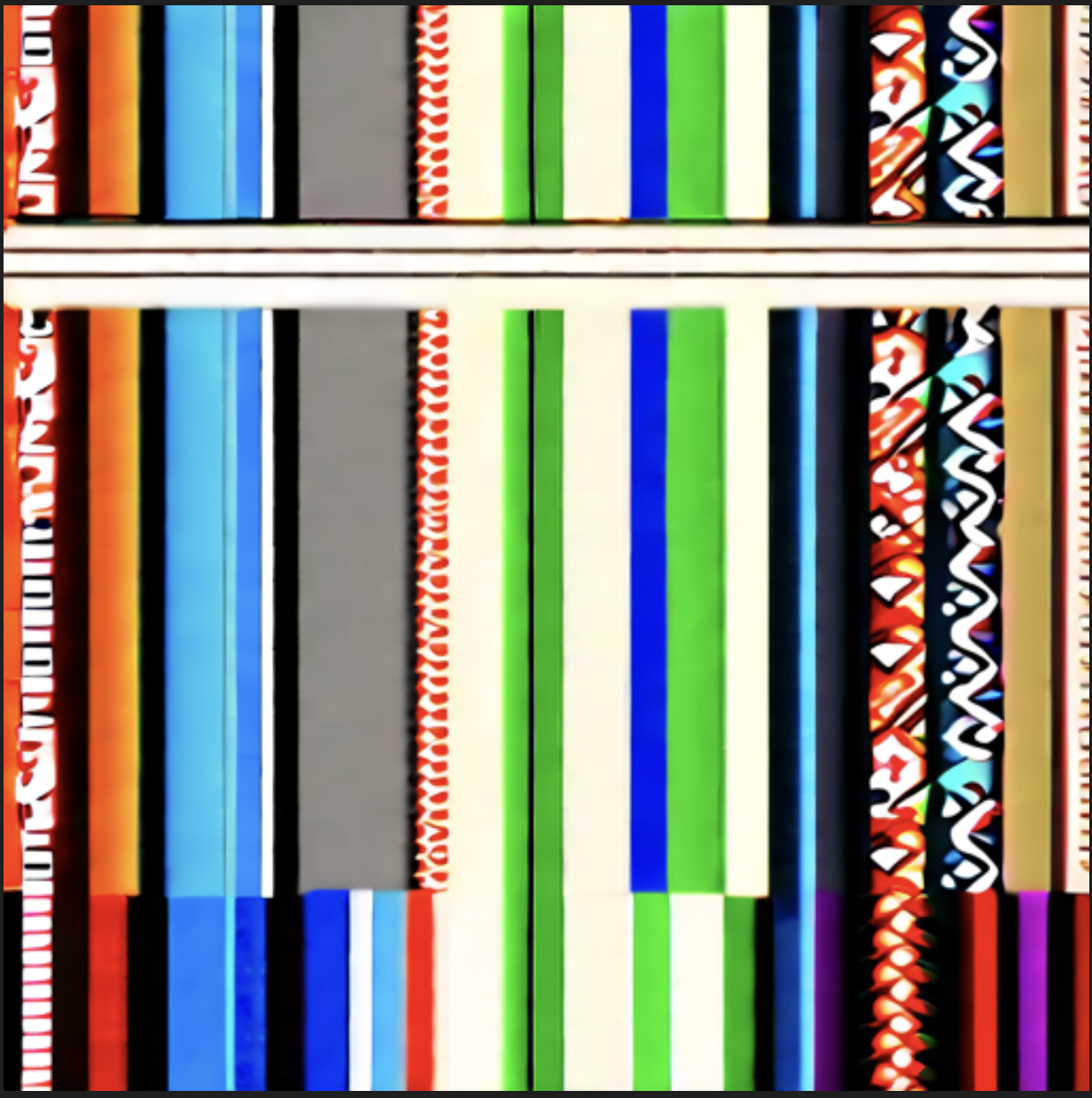}   
    \caption{\textit{\textbf{Impact of Repetitive Patches in the Initial Noise on Image Generation.}}}
    \label{fig:carpet}
\end{figure}

\section{Additional Experiments}

\subsection{CatAttack Performance vs. Object Scale}
We varied the size of the pasted object in the CatAttack from 10\% to 40\% of the image area and evaluated detection performance at each scale. Table~\ref{tab:cat_attack_app} reports the ROC-AUC (\%) for each object scale, showing a gradual improvement from 95.4\% at a 10\% scale to 98.0\% at a 40\% scale.

\begin{table}[h]
\centering
\caption{\textit{\textbf{CatAttack Detection Performance vs. Object Scale.}} ROC-AUC (\%) for different object sizes, as detected by SEAL + Spatial Test.}
\label{tab:cat_attack_app}
\small
\begin{tabular}{lcccc}
\toprule
Scale (\% of image) & 10 & 20 & 30 & 40 \\
\midrule
ROC-AUC (\%)        & 95.4 & 96.2 & 97.7 & 98.0 \\
\bottomrule
\end{tabular}
\end{table}

\subsection{Effect of Insertions on Semantic Embeddings and LSH Binning}
\label{app:bin_switch}
We evaluate how localized insertions (e.g., “cat”, “house”, “human”) impact the semantic embedding vector and its SimHash bin assignments. Table~\ref{tab:insert_analysis} reports the average angular shift $\Delta \theta$ between the original and edited embeddings, along with the proportion of hash bins that flip due to each insertion. The results show that even modest insertions can produce substantial rotations in semantic space (up to $71.2\degree$) and high bin-flip ratios (exceeding $90\%$), underscoring the sensitivity of LSH-based watermark detection to semantic changes. The experimental details are similar to those of \textit{Cat Attack} in \Cref{sec:experiments}.

\begin{table}[h]
    \centering
    \caption{\textit{\textbf{Insertion Type Analysis.} The angle between the original and edited semantic embeddings, and the ratio of changed bins.}}
    \small
\begin{tabular}{lccc}
\toprule
Insert Type & Cat & House & Human \\
\midrule
Angle ($^\circ$) & $71.2 \pm 13.8$ & $60.3 \pm 19.7$ & $68.8 \pm 17.3$ \\
Flip Ratio & $96\% \pm 4\%$ & $90\% \pm 8\%$ & $94\% \pm 5\%$ \\
\bottomrule
\end{tabular}
\label{tab:insert_analysis}
\end{table}

\subsection{Robustness under Regeneration Attack}
\label{app:regenration_attacks}
We benchmark SEAL against a range of detection approaches, including both generation-time and post-hoc methods such as HiDDeN~\cite{zhu2018hidden}, Stable Signature~\cite{fernandez2023stable}, TrustMark~\cite{bui2023trustmarkuniversalwatermarkingarbitrary}, and WOUAF~\cite{kim2024wouafweightmodulationuser}. As shown in Table~\ref{tab:regen-robustness}, SEAL maintains a 98\% detection accuracy under a regeneration attack~\cite{zhao2025invisible}, outperforming prior methods.

\begin{table}[htbp]
  \centering
  \caption{Robustness of watermarking methods under regeneration-based removal attacks \cite{zhao2025invisible}. We report the watermark detection accuracy on regenerated (attacked) images (\%).}
  \label{tab:regen-robustness}
  \scriptsize
  \begin{tabular}{lccccc}
    \toprule
    Method      & HiDDeN & StableSig & TrustMark & WOUAF & SEAL \\
    \midrule
    Acc.\ (\%)  & 47     & 41        & 5         & 51    & \textbf{98}   \\
    \bottomrule
  \end{tabular}
\end{table}

\subsection{Proxy vs.\ Generated Image Comparison}
In~\Cref{fig:proxy_vs_actual}, we present a visual comparison between the proxy image \(x_{\text{pre}}\), which guides the watermark placement, and the final generated image \(x\), illustrating SEAL’s ability to preserve semantic intent during watermarked generation.

\begin{figure}[t]
    \centering
    \begin{subfigure}[b]{0.24\linewidth}
      \includegraphics[width=\linewidth]{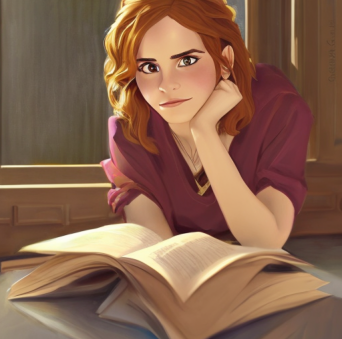}
    \end{subfigure}
    \hfill
    \begin{subfigure}[b]{0.24\linewidth}
      \includegraphics[width=\linewidth]{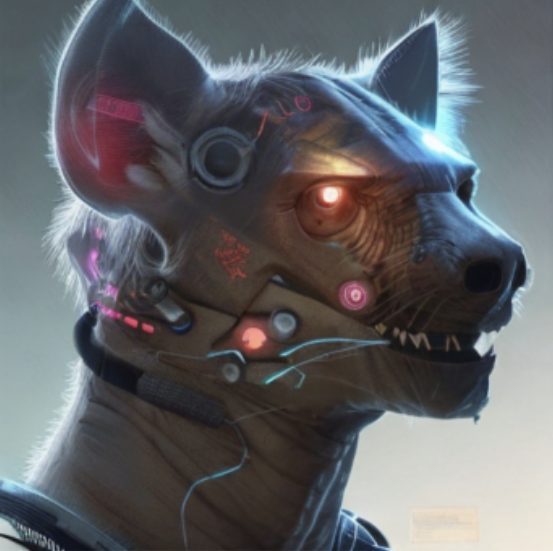}
    \end{subfigure}
    \hfill
    \begin{subfigure}[b]{0.24\linewidth}
      \includegraphics[width=\linewidth]{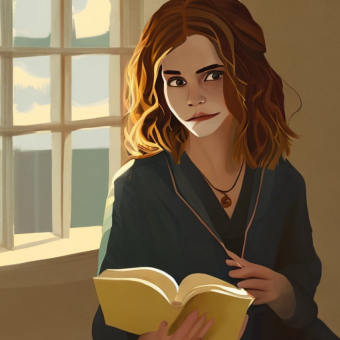}
    \end{subfigure}
    \hfill
    \begin{subfigure}[b]{0.24\linewidth}
      \includegraphics[width=\linewidth]{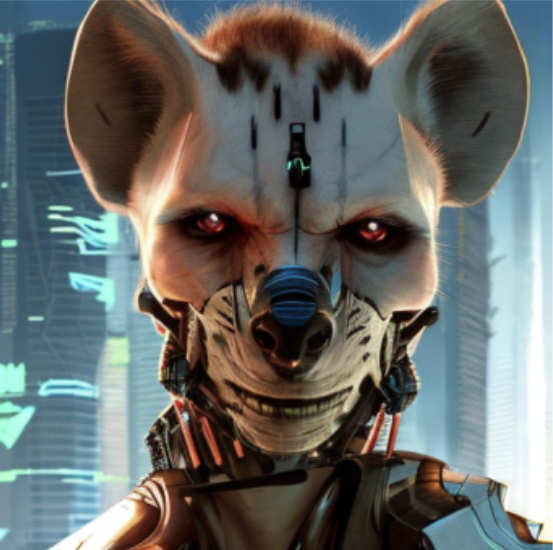}
    \end{subfigure}
    \caption{Comparison between the proxy image \(x_{\text{pre}}\) (left two) and the final generated image \(x\) (right two). Zoom in for clarity.}
    \label{fig:proxy_vs_actual}
\end{figure}



\begin{table*}[t!]
\centering
\caption{\textit{\textbf{CLIP Score Evaluation.} Comparison of CLIP scores before and after watermarking for images generated using prompts from the Stable-Diffusion-Prompts~\cite{gustavosta2024pro} and COCO~\cite{lin2014microsoft} dataset.}}
\begin{tabular}{cccc}
\toprule
\multicolumn{2}{c}{Stable-Diffusion-Prompts} & \multicolumn{2}{c}{COCO} \\
\cmidrule(lr){1-2} \cmidrule(lr){3-4}
CLIP (before) & CLIP (after) & CLIP (before) & CLIP (after) \\
\midrule
32.378 & 32.401 & 31.365 & 31.499 \\
\bottomrule
\end{tabular}
\label{tab:clip_results}
\end{table*}

\begin{figure*}[h]
  \centering
  \includegraphics[width=0.8\textwidth]{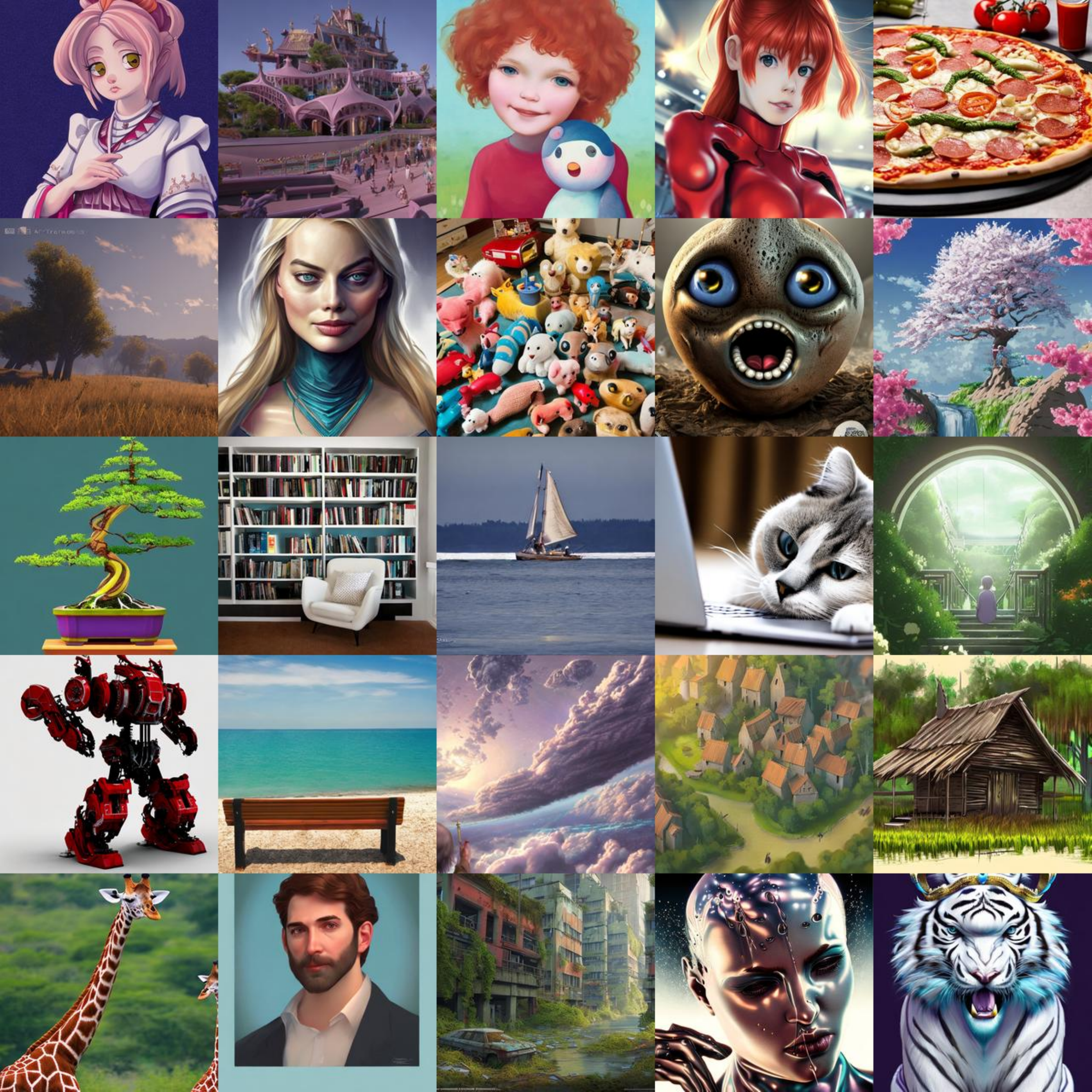}
  \caption{\textbf{\textit{Watermarked images generated using \methodname.}}}
  \label{fig:generated_images}
\end{figure*}

\end{document}